\documentclass[letterpaper]{article}
\usepackage{uai2020}
\usepackage[margin=1in]{geometry}
\usepackage{microtype}
\usepackage{graphicx}
\usepackage{subfigure}
\usepackage{booktabs}

\usepackage{hyperref}

\usepackage{url}
\usepackage{amssymb,amsmath}
\usepackage{algorithm,algorithmic}
\usepackage{graphicx}

\usepackage{amsthm}

\newtheorem{prop}{Proposition}

\usepackage{stix}
\usepackage{natbib}

\title{Joint Stochastic Approximation and Its Application to Learning Discrete Latent Variable Models}
\author{ {\bf Zhijian Ou} \\
Electronic Engineering Department\\
Tsinghua University, Beijing, China\\
\And
{\bf Yunfu Song}  \\
Electronic Engineering Department\\
Tsinghua University, Beijing, China\\
}

\begin{document}
\maketitle

%

\newcommand{\fix}{\marginpar{FIX}}
\newcommand{\new}{\marginpar{NEW}}



\begin{abstract}
Although with progress in introducing auxiliary amortized inference models, learning discrete latent variable models is still challenging. In this paper, we show that the annoying difficulty of obtaining reliable stochastic gradients for the inference model and the drawback of indirectly optimizing the target log-likelihood can be gracefully addressed in a new method based on stochastic approximation (SA) theory of the Robbins-Monro type.
Specifically, we propose to directly maximize the target log-likelihood and simultaneously minimize the inclusive divergence between the posterior and the inference model.
The resulting learning algorithm is called joint SA (JSA).
To the best of our knowledge, JSA represents the first method that couples an SA version of the EM (expectation-maximization) algorithm (SAEM) with an adaptive MCMC procedure.
Experiments on several benchmark generative modeling and structured prediction tasks show that JSA consistently outperforms recent competitive algorithms, with faster convergence, better final likelihoods, and lower variance of gradient estimates.
\end{abstract}

\section{INTRODUCTION}

A wide range of machine learning tasks involves observed and unobserved data.
Latent variable models explain observations as part of a partially observed system and usually express a joint distribution $p_\theta(x,h)$ over observation $x$ and its unobserved counterpart $h$, with parameter $\theta$.
Models with discrete latent variables are broadly applied, including mixture modeling, unsupervised learning \citep{SBN,nvil}, structured output prediction \citep{Raiko2014TechniquesFL,vimco} and so on.

Currently variational methods are widely used for learning latent variable models, especially those parameterized using neural networks.
In such methods, an auxiliary amortized inference model $q_\phi(h|x)$ with parameter $\phi$ is introduced to approximate the posterior $p_\theta(h|x)$ \citep{vae,Rezende2014StochasticBA}, and some bound of the marginal log-likelihood, used as a surrogate objective, is optimized over both $\theta$ and $\phi$. Two well-known bounds are the evidence lower bound (ELBO) \citep{ELBO} and the multi-sample importance-weighted (IW) lower bound \citep{IWAE}.
Though with progress (as reviewed in Section \ref{sec:related-work}), a difficulty in variational learning of discrete latent variable models is to obtain reliable (unbiased, low-variance) Monte-Carlo estimates of the gradient of the bound with respect to (w.r.t.) $\phi$\footnote{since the gradient of the ELBO w.r.t. $\theta$ usually can be reliably estimated.}.

Additionally, a common drawback in many existing methods for learning latent-variable models is that they indirectly optimize some bound of the target marginal log-likelihood. This leaves an uncontrolled gap between the marginal log-likelihood and the bound, depending on the expressiveness of the inference model.
There are hard efforts to develop more expressive but increasingly complex inference models to reduce the gap \citep{Salimans2014BridgingGap,Rezende2015NormalizingFlows,Maale2016AuxiliaryDG,Kingma2017ImprovedVI}\footnote{Notably, some methods are mainly applied to continuous latent variables, e.g. it is challenging to directly apply normalizing flows to discrete random variables, though recently there are some effort \citep{Ziegler2019LatentNF}.}.
But it is highly desirable that we can eliminate the effect of the gap on model learning or ideally directly optimize the marginal log-likelihood, without increasing the complexity of the inference model.

In this paper, we show that the annoying difficulty of obtaining reliable stochastic gradients for the inference model and the drawback of indirectly optimizing the target log-likelihood can be gracefully addressed in a new method based on stochastic approximation (SA) theory of the Robbins-Monro type \citep{SA51}.
These two seemingly unrelated issues are inherently related to our choice that we optimizes the ELBO or the similar IW lower bound.

Specifically, we propose to directly maximize w.r.t. $\theta$ the marginal log-likelihood\footnote{``directly'' is in the sense that we set the marginal log-likelihood as the objective function in stochastic optimization.} and simultaneously minimize w.r.t. $\phi$ the \emph{inclusive} divergence\footnote{$KL[p_\theta||q_\phi] \triangleq \int p_\theta \log \left( p_\theta / q_\phi \right)$} $KL[p_\theta(h|x)||q_\phi(h|x)]$ between the posterior and the inference model, and fortunately, we can use the SA framework to solve the joint optimization problem.
The key is to recast the two gradients as expectations and equate them to zero; then the equations can be solved by applying the SA algorithm, in which the inference model serves as an adaptive proposal for constructing the Markov Chain Monte Carlo (MCMC) sampler.
The resulting learning algorithm is called joint SA (JSA), as it couples SA-based model learning and SA-based adaptive MCMC and jointly finds the two sets of unknown parameters ($\theta$ and $\phi$).

It is worthwhile to recall that there is an another class of methods in learning latent variable models for maximum likelihood (ML), even prior to the recent development of variational learning for approximate ML, which consists of the expectation-maximization (EM) algorithm \citep{Dempster1977MaximumLF} and its extensions.
Interestingly, we show that the JSA method amounts to coupling an SA version of EM (SAEM) \citep{delyon1999convergence,kuhn2004coupling} with an adaptive MCMC procedure. This represents a new  extension among the various stochastic versions of EM in the literature. This revealing of the connection between JSA and SAEM is important for us to appreciate the new JSA method.

The JSA learning algorithm can handle both continuous and discrete latent variables. The application of JSA in the continuous case is not pursued in the present work, and we leave it as an avenue of future exploration.
In this paper, we mainly present experimental results for learning discrete latent variable models with Bernoulli and categorical variables, consisting of stochastic layers or neural network layers.
Our results on several benchmark generative modeling and structured prediction tasks demonstrate that JSA consistently outperforms recent competitive algorithms, with faster convergence, better final likelihoods, and lower variance of gradient estimates.

\section{PRELIMINARIES}

\subsection{STOCHASTIC APPROXIMATION (SA)}

\begin{algorithm}[tb]
	\caption{The general stochastic approximation (SA) algorithm}\label{alg:SA}
	\begin{algorithmic}	
		\FOR {$t=1,2,\cdots$}
		\STATE \underline{Monte Carlo sampling:} Draw a sample $z^{(t)}$ with a Markov transition kernel $K_{\lambda^{(t-1)}}(z^{(t-1)},\cdot)$, which starts with $z^{(t-1)}$ and admits $p_{\lambda^{(t-1)}}(\cdot)$ as the invariant distribution.
		\STATE \underline{SA updating:} Set $\lambda^{(t)} = \lambda^{(t-1)} + \gamma_t F_{\lambda^{(t-1)}}(z^{(t)})$, where $\gamma_t$ is the learning rate.
		\ENDFOR
	\end{algorithmic}
\end{algorithm}

Stochastic approximation methods are an important family of iterative stochastic optimization algorithms, introduced in \citep{SA51} and extensively studied \citep{benveniste2012adaptive,chen2002stochastic}.
Basically, stochastic approximation provides a mathematical framework for stochastically solving a root finding problem, which has the form of expectations being equal to zeros.
Suppose that the objective is to find the solution $\lambda^*$ of $f(\lambda) = 0$ with
\begin{equation}
\label{eq:SA}
f(\lambda) = E_{z \sim p_\lambda(\cdot) } [ F_\lambda(z) ],
\end{equation}
where $\lambda$ is a $d$-dimensional parameter vector, and $z$ is an observation from a probability distribution $p_\lambda(\cdot)$ depending on $\lambda$.
$F_\lambda(z) \in R^d $ is a function of $z$, providing $d$-dimensional stochastic measurements of the so-called mean-field function $f(\lambda)$.
Intuitively, we solve a system of simultaneous equations, $f(\lambda) = 0$, which consists of $d$ constraints, for determining $d$-dimensional $\lambda$.

Given some initialization $\lambda^{(0)}$ and $z^{(0)}$, a general SA algorithm iterates Monte Carlo sampling and parameter updating, as shown in Algorithm \ref{alg:SA}.
The convergence of SA has been established under conditions \citep{benveniste2012adaptive, andrieu2005stability, song2014weak}, including a few technical requirements for the mean-field function $f(\lambda)$, the Markov transition kernel $K_{\lambda^{(t-1)}}(z^{(t-1)},\cdot)$ and the learning rates.
Particularly, when $f(\lambda)$ corresponds to the gradient of some objective function, then $\lambda^{(t)}$ will converge to local optimum, driven by stochastic gradients  $F_\lambda(z)$.
To speed up convergence, during each SA iteration, it is possible to generate a set of multiple observations $z$ by performing the Markov transition repeatedly 
and then use the average of the corresponding values of $F_\lambda(z)$ for updating $\lambda$, which is known as SA with multiple moves \citep{Wang2017LearningTR}, as shown in Algorithm \ref{alg:SA-multiple-move} in Appendix.

Remarkably, Algorithm \ref{alg:SA} shows stochastic approximation with Markovian perturbations \citep{benveniste2012adaptive}. It is more general than the non-Markovian SA which requires exact sampling $z^{(t)} \sim p_{\lambda^{(t-1)}}(\cdot)$ at each iteration and in some tasks can hardly be realized.
In non-Markovian SA, we check that $F_\lambda(z)$ is unbiased estimates of $f(\lambda)$, while in SA with Markovian perturbations, we check the ergodicity property of the Markov transition kernel.

\subsection{VARIATIONAL LEARNING METHODS} \label{sec:vae}
Here we briefly review the variational methods, recently developed for learning latent variable models \citep{vae,Rezende2014StochasticBA}.
Consider a latent variable model $p_\theta(x,h)$ for observation $x$ and latent variable $h$, with parameter $\theta$.
Instead of directly maximizing the marginal log-likelihood $\log p_\theta(x)$ for the above latent variable model, variational methods maximize the variational lower bound (also known as ELBO), after introducing an auxiliary amortized inference\footnote{$q_\phi(h|x)$ uses a single, global set of parameters $\phi$ over the entire training set, which is called amortized inference.} model $q_\phi(h|x)$:
\begin{displaymath}
\begin{split}
ELBO(\theta,\phi;x) &\triangleq E_{q_\phi(h|x)}\log \frac{p_\theta(x,h)}{q_\phi(h|x)}\\
&=\log p_\theta(x) - KL\left[ q_\phi(h|x) || p_\theta(h|x) \right]
\end{split}
\end{displaymath}
It is known that maximizing the ELBO w.r.t. ${\phi}$ amounts to minimize the exclusive KL-divergence $KL[q_{{\phi}}({h}|{x})|| p_{{\theta}}({h}|{x})]$, which has the annoying effect of high variance in estimating gradients mentioned before.

\section{METHOD}

\subsection{DEVELOPING JSA}

Consider a latent variable model $p_\theta(x,h)$ for observation $x$ and latent variable $h$, with parameter $\theta$.
Like in variational methods, 
we also jointly train the target model $p_\theta(x,h)$ together with an auxiliary amortized inference model $q_\phi(h|x)$.
The difference is that we propose to directly maximize w.r.t. $\theta$ the marginal log-likelihood and simultaneously minimizes w.r.t. $\phi$ the inclusive KL divergence $KL(p_\theta(h|x)||q_\phi(h|x))$ between the posterior and the inference model, pooled over the empirical dataset:
\begin{equation}
\label{eq:JSA_unsup_obj}
\left\{
\begin{split}
& \min_{\theta} KL\left[ \tilde{p}(x) || p_\theta(x) \right] \\
& \min_{\phi} E_{\tilde{p}(x)} KL\left[ p_\theta(h|x)|| q_\phi(h|x) \right] \\
\end{split}
\right.
\end{equation}
where $\tilde{p}(x) \triangleq \frac{1}{n} \sum_{i=1}^{n} \delta(x-x_i)$ denotes the empirical distribution for a training dataset consisting of $n$ independent and identically distributed (IID) data points $\left\lbrace x_1, \cdots, x_n \right\rbrace $.
In such a way, we pursue direct maximum likelihood estimation of $\theta$ and at the same time avoid the high-variance gradient estimate w.r.t. $\phi$.

Fortunately, we can use the SA framework to solve the joint optimization problem Eq.(\ref{eq:JSA_unsup_obj}) as described below.
First, it can be shown that the gradients for optimizing the two objectives in Eq. (\ref{eq:JSA_unsup_obj}) can be derived as follows\footnote{The first equation in Eq.(\ref{eq:SA}) directly follows from the Fisher identity Eq.(\ref{eq:fisher}).}:
\begin{equation}
\label{eq:JSA_unsup_gradient}
\left\{
\begin{split}
g_\theta \triangleq -\nabla_\theta KL[ \tilde{p}&(x) || p_\theta(x) ]\\
&=E_{\tilde{p}(x) p_\theta(h|x)}\left[\nabla_\theta logp_\theta(x,h)\right] \\
g_\phi \triangleq -\nabla_\phi  E_{\tilde{p}(x)} &KL\left[ p_\theta(h|x)|| q_\phi(h|x) \right]\\
&=E_{\tilde{p}(x) p_\theta(h|x)}\left[ \nabla_\phi logq_\phi(h|x)\right]
\end{split}
\right.
\end{equation}
By the following Proposition \ref{prop:jsa_expectation_form}, Eq.(\ref{eq:JSA_unsup_gradient}) can be recast in the expectation form of Eq.(\ref{eq:SA}).
The optimization problem can then be solved by setting the gradients to zeros and applying the SA algorithm to finding the root for the resulting system of simultaneous equations.
\begin{prop} \label{prop:jsa_expectation_form}
	The gradients w.r.t. $\theta$ and $\phi$ as in Eq.(\ref{eq:JSA_unsup_gradient}) can be recast in the expectation form of Eq.(\ref{eq:SA}) (i.e. as expectation of stochastic gradients), by letting $\lambda \triangleq (\theta, \phi)^T$, $z \triangleq (\kappa, h_1,\cdots,h_n)^T$, $p_\lambda(z) \triangleq \frac{1}{n} \prod_{i=1}^{n} p_\theta(h_i|x_i)$, $f(\lambda) \triangleq (g_\theta, g_\phi)^T$, and
	\begin{displaymath}
	F_\lambda(z) \triangleq \left( \begin{array}{c}
	\sum_{i=1}^{n} \delta(\kappa=i) \nabla_\theta logp_\theta(x_i,h_i) \\
	\sum_{i=1}^{n} \delta(\kappa=i) \nabla_\phi logq_\phi(h_i|x_i)
	\end{array} \right).
	\end{displaymath}
	In order to avoid visiting all $h_1,\cdots,h_n$ at every SA iteration (to be explained below), we introduce an index variable $\kappa$ which is uniformly distributed over $1,\cdots,n$. 
	$\delta(\kappa=i)$ denotes the indicator of $\kappa$ taking $i$.
\end{prop}
\begin{proof}
	This can be readily seen by rewriting Eq.(\ref{eq:JSA_unsup_gradient}) as:
	\begin{equation} \label{eq:JSA_unsup_gradient_batch}
	\left\{
	\begin{split}
	& g_\theta = \frac{1}{n} \sum_{i=1}^{n} E_{p_\theta(h_i|x_i)}\left[\nabla_\theta logp_\theta(x_i,h_i)\right]\\
	& g_\phi = \frac{1}{n} \sum_{i=1}^{n} E_{p_\theta(h_i|x_i)}\left[ \nabla_\phi logq_\phi(h_i|x_i)\right]
	\end{split}
	\right.
	\end{equation}
	and applying $E [\delta(\kappa=i)] = \frac{1}{n}$ and the independence between $\kappa, h_1,\cdots,h_n$.
\end{proof}

Recall that as defined in Proposition \ref{prop:jsa_expectation_form}, $z$ consists of $n+1$ independent components, $\kappa$, and $h_1,\cdots,h_n$.
At iteration $t$, the SA algorithm need to draw sample $z^{(t)} \sim p_{\lambda^{(t-1)}}(z)$ either directly or through a Markov transition kernel, and then to update parameters based on stochastic gradients $F_{\lambda^{(t-1)}}(z^{(t)})$ calculated using this sample $z^{(t)}$.
The introduction of $\kappa$ allows us to calculate the stochastic gradients for only one $h_i$, indexed by $\kappa$, instead of calculating over the full batch of all $n$ data points (as originally suggested by Eq.(\ref{eq:JSA_unsup_gradient_batch})).
Minibatching (i.e. drawing a subset of data points) can be achieved by running SA with multiple moves, once the SA procedure with only one $h_i$ is established (see Appendix \ref{sec:additional-appendix} for details).

The introduction of $\kappa$ also allows us to avoid drawing a new full set of $h_1,\cdots,h_n$ in sampling $z^{(t)} \sim p_{\lambda^{(t-1)}}(z)$ at every SA iteration.
Suppose that we can construct a base transition kernel for each $h_i$, denoted by $B_{\lambda,i}(h_i^{(t-1)},\cdot)$, which admits $p_\theta(h_i|x_i)$ as the invariant distribution, $i=1,\cdots,n$.
Specifically, we can first draw $\kappa$ uniformly over $1,\cdots,n$, and then draw $h_\kappa$ by applying $B_{\lambda,\kappa}(h_\kappa^{(t-1)},\cdot)$ and leave the remaining components $h_1,\cdots,h_{\kappa-1}, h_{\kappa+1},\cdots,h_n$ unchanged.
This is a special instance\footnote{We just sample according to the marginal distribution of $h_i, i=1,\cdots,n$, instead of the conditional distribution of $h_i$, since they are independent.} of random-scan Metropolis-within-Gibbs sampler \citep{Fort2003OnTG}.

For designing the base transition kernel $B_{\lambda,i}(h_i^{(t-1)},\cdot)$, $i=1,\cdots,n$, we propose to utilize the auxiliary inference model $q_\phi(h_i|x_i)$ as a proposal and use the Metropolis independence sampler (MIS) \citep{LiuJun2008}, targeting $p_\theta(h_i|x_i)$.
Given current sample $h_i^{(t-1)}$, MIS works as follows:
\begin{enumerate}
	\item Propose $h_i \sim q_\phi(h_i|x_i)$;
	\item Accept $h_i^{(t)}=h_i$ with prob.
	$	min\left\lbrace 1,  \frac{w(h_i)}{w( h_i^{(t-1)})} \right\rbrace$,
\end{enumerate}
where $w(h_i) = \frac{p_\theta(h_i|x_i)}{q_\phi(h_i|x_i)}$ is the importance sampling weight.
Remarkably, we can cancel out the intractable $p_{\theta}(x_i)$ which appears in both the numerator $w(h_i)$ and denominator $w(h^{(t-1)}_i)$. 
The Metropolis-Hastings (MH) ratio is then calculated as $\frac{p_\theta(h_i,x_i)}{q_\phi(h_i|x_i)} / \frac{p_\theta(h^{(t-1)}_i,x_i)}{q_\phi(h^{(t-1)}_i|x_i)}$.

Finally, the JSA algorithm is summarized in Algorithm \ref{alg:JAE}.

\begin{algorithm}[tb]
	\caption{The JSA algorithm}
	\label{alg:JAE}
	\begin{algorithmic}
		\REPEAT
		\STATE \underline{Monte Carlo sampling:}\\
		Draw $\kappa$ over $1,\cdots,n$, pick the data point $x_\kappa$ along with the cached $h^{(old)}_\kappa$, and use MIS to draw $h_\kappa$;
		\STATE \underline{SA updating:}\\
		Update $\theta$ by ascending: $\nabla_\theta logp_\theta(x_\kappa,h_\kappa)$;\\
		Update $\phi$ by ascending: $\nabla_\phi logq_\phi(h_\kappa|x_\kappa)$;
		\UNTIL{convergence}
	\end{algorithmic}
\end{algorithm}

\subsection{CONNECTING JSA TO MCMC-SAEM}
In this section, we reveal that the JSA algorithm amounts to coupling an SA version of EM (SAEM) \citep{delyon1999convergence,kuhn2004coupling} with an adaptive MCMC procedure.
To simplify discussion, we consider learning with a single training data point. Learning with a set of IID training data points can be similarly analyzed.

The EM algorithm is an iterative method to find maximum likelihood estimates of parameters for latent variable models (or in statistical terminology, from incomplete data). At iteration $t$, the E-step calculates the Q-function $Q(\theta|\theta^{(t-1)}) =  E_{p_{\theta^{(t-1)}}(h|x)}\left[ \nabla_\theta logp_\theta(x,h)\right]$ and the M-step updates $\theta$ by maximizing $Q(\theta|\theta^{(t-1)})$ over $\theta$ or performing gradient ascent over $\theta$ when a closed-form solution is not available. 
In the E-step, when the expectation in $Q(\theta|\theta^{(t-1)})$ cannot be tractably evaluated, SAEM has been developed \citep{delyon1999convergence}.

SAEM can be readily obtained from the Fisher identity (see Appendix \ref{sec:Fisher} for proof),
\begin{equation}
\label{eq:fisher}
\nabla_\theta log p_\theta(x) = E_{p_\theta(h|x)}\left[ \nabla_\theta logp_\theta(x,h)\right],
\end{equation}
which again can be viewed as in the expectation form of Eq.(\ref{eq:SA}) (i.e. as expectation of stochastic gradients). So we can apply the SA algorithm, which yields SAEM.
It can be seen that the Monte Carlo sampling step in SA fills the missing values for latent variables through sampling $h' \sim p_\theta(h|x)$, which is analogous to the E-step in EM.
The SA updating step performs gradient ascent over $\theta$ using $\nabla_\theta logp_\theta(x,h')$, analogous to the M-step in EM.

When exact sampling from $p_{\theta^{(t-1)}}(h|x)$ is difficult, an MCMC-SAEM algorithm has been developed \citep{kuhn2004coupling}. MCMC-SAEM draws a sample of the latent $h$ by applying a Markov transition kernel which admits $p_{\theta^{(t-1)}}(h|x)$ as the invariant distribution.
Given $\theta^{(t-1)}$, the MCMC step in classic MCMC-SAEM is non-adaptive in the sense that the proposal of the  transition kernel is fixed.
In contrast, in JSA, the auxiliary amortized inference model $q_\phi(h|x)$ acts like an adaptive proposal, adjusted from past realizations of the Markov chain, so that the Markov transition kernel is adapted.

\section{RELATED WORK}
\label{sec:related-work}

\textbf{Novelty.}
MCMC-SAEM \citep{kuhn2004coupling} and adaptive MCMC \citep{andrieu2006ergodicity} have been separately developed in the SA framework.
To the best of our knowledge, JSA represents the first method that couples MCMC-SAEM with adaptive MCMC, and encapsulate them through a joint SA procedure.
The model learning of $\theta$ and the proposal tuning through $\phi$ are coupled, evolving together to converge (see Appendix \ref{sec:convergence_jsa} for convergence of JSA).
This coupling has important implications for reducing computational complexity (see below) and improving the performance of MCMC-SAEM with adaptive proposals.

Depending on the objective functions used in joint training of the latent variable model $p_\theta$ and the auxiliary inference model $q_\phi$, JSA is also distinctive among existing methods.
A class of methods relies on a single objective function, either the variational lower bound (ELBO) or the IW lower bound, for optimizing both $p_\theta$ and $q_\phi$, but suffers from seeking reliable stochastic gradients for the inference model.
The reweighted wake-sleep (RWS) algorithm uses the IW lower bound for optimizing $p_\theta$ and the inclusive divergence for optimizing $q_\phi$ \citep{rws}.
In contrast, JSA directly optimizes the marginal log-likelihood for $p_\theta$. Though both RWS and JSA use the inclusive divergence for optimizing $q_\phi$, the proposed samples from $q_\phi(h|x)$ are always accepted in RWS, which clearly lacks an accept/reject mechanism by MCMC for convergence guarantees.

Compared to the earlier work \citep{xu2016joint}, this paper is a new development by introducing the random-scan sampler and a formal proof of convergence (Proposition \ref{prop:jsa_expectation_form} and Appendix \ref{sec:convergence_jsa}). A similar independent work pointed out by one of the reviewers is called Markov score climbing (MSC) \citep{naesseth2020markovian} (see Appendix \ref{sec:additional-appendix} for differences between JSA and MSC).

\textbf{Computational complexity.~}
It should be stressed that though we use MCMC to draw samples from the posterior $p_\theta(h|x)$, we do not need to run the Markov chain for sufficiently long time to converge within one SA iteration\footnote{To understand this intuitively, first, note that the inference model is adapted to chase the posterior on the fly, so the proposed samples from $q_\phi(h|x)$ are already good approximate samples for the posterior. Second, the chain could be viewed as running continuously across iterations and dynamically close to the slowly-changing stationary distribution.
}.
This is unlike in applications of MCMC solely for inference.
In learning latent variable models, JSA only runs a few steps of transitions (the same as the number of samples drawn in variational methods) per parameter update and still have parameter convergence. Thus the training time complexity of JSA is close to that of variational methods.

One potential problem with JSA is that we need to cache the sample for latent $h$ per training data point in order to make a persistent Markov Chain.
Thus JSA trades storage for quality of model learning. This might not be restrictive given today's increasingly large and cheap storage.

\textbf{Minimizing inclusive divergence~}
$KL[p_\theta(h|x)||$
$q_\phi(h|x)]$ 
w.r.t. $\phi$ ensures that $q_\phi(h|x)>0$ wherever $p_\theta(h|x)>0$, which is a basic restriction that makes $q_\phi(h|x)$ a valid proposal for MH sampling of $p_\theta(h|x)$. The exclusive divergence is unsuitable for this restriction. 
Inclusive minimization also avoid the annoying difficulty of obtaining reliable stochastic gradients for $\phi$, which is suffered by minimizing the exclusive divergence.

\textbf{Learning discrete latent variable models.~}
In order to obtain reliable (unbiased, low-variance) Monte-Carlo estimates of the gradient of the bound w.r.t. $\phi$, two well-known possibilities are continuous relaxation of discrete variable \citep{concrete,gumbel-softmax}, which gives low-variance but biased gradient estimates, and the REINFORCE method \citep{reinforce}, which yields unbiased but high-variance gradient estimates. Different control variates - NVIL \citep{nvil}, VIMCO \citep{vimco}, MuProp \citep{gu2015muprop}, REBAR \citep{tucker2017rebar}, RELAX \citep{grathwohl2018backpropagation}, are developed to reduce the variance of REINFORCE.
Other recent possibilities include a Rao-Blackwellization procedure \citep{Liu2018RaoBlackwellizedSG}, a finite difference approximation \citep{Lorberbom2018DirectOT}, a gradient reparameterization via Dirichlet noise \citep{ARSM} or sampling without replacement \citep{Kool}.

\textbf{The gap between the marginal log-likelihood and the ELBO} is related to the mismatch between the inference model and the true posterior.
Some studies develop more expressive but increasingly complex inference models, which can also be used in JSA.
To reduce the gap without introducing any additional complexity to the inference model, there are methods to seek tighter bound, e.g. the IW bound in \citep{IWAE}. But it is shown in \citep{rainforth2018tighter} that using tighter bound of this form is detrimental to the process of learning the inference model. 
There are efforts to incorporate MCMC into variational inference to reduce the gap.
\citep{Salimans2014BridgingGap} introduces some auxiliary variables from a Markov chain in defining a tighter ELBO, but with additional neural networks for forward and backward transition operators.
\citep{hoffman2017learning} seeks ML estimate of $\theta$ by performing additional Hamiltonian Monte Carlo (HMC) steps from the proposal given by $q_\phi$, but still estimate $\phi$ by maximizing ELBO. This method only works with continuous latent variables by using HMC and would be limited since minimizing exclusive divergence encourages low entropy inference models which are not good for proposals.
In contrast, JSA is not severely suffered by the mismatch, since although the mismatch affects the sampling efficiency of MIS and SA convergence rate, we still have theoretical convergence to ML estimate.

\textbf{Adaptive MCMC~}
is an active research area. A classic example is adaptive scaling of the variance of the Gaussian proposal in random-walk Metropolis \citep{roberts2009examples}.
Recently, some adaptive MCMC algorithms are analyzed as SA procedures to study their ergodicity properties \citep{andrieu2006ergodicity}.
These algorithms minimize the inclusive divergence between the target distribution and the proposal, but use a mixture of distributions from an exponential family as the adaptive proposal.
The L2HMC \citep{Levy2017GeneralizingHM} learns a parametric leapfrog operator to extend HMC mainly in continuous space, by maximizing expected squared jumped distance.
The auxiliary variational MCMC \citep{habib2018auxiliary} optimizes the proposal by minimizing exclusive divergence.

\section{EXPERIMENTS}
In this section, we evaluate the effectiveness of the proposed JSA method in training models for generative modeling with Bernoulli and categorical variables and structured output prediction with Bernoulli variables. 
\subsection{BERNOULLI LATENT VARIABLES}
\label{sec:generative-ber}
We start by applying various methods to training generative models with Bernoulli (or binary) latent variables,
comparing JSA to REINFORCE \citep{Williams1992SimpleSG}, NVIL \citep{nvil}, RWS \citep{rws}, Concrete \citep{concrete}/Gumbel-Softmax \citep{gumbel-softmax}, REBAR \citep{tucker2017rebar}, VIMCO \citep{vimco}, and ARM\footnote{ARSM with Bernoulli latent variables reduces to ARM.} \citep{arm}.
For JSA, RWS and VIMCO, we use \textit{particle-number} $=2$ (i.e. computing gradient with 2 Monte Carlo samples during training), which yields their theoretical time complexity comparable to ARM.

We follow the model setup in \citep{concrete}, which is also used in \citep{arm}.
For $q_\phi(h|x)$ and $p_\theta(x|h)$, three different network architectures are examined, as summarized in Table \ref{tab:ber-net} in Appendix, including ``Nonlinear'' that has one stochastic but two Leaky-ReLU deterministic hidden layers, ``Linear'' that has one stochastic hidden layer, and ``Linear two layers'' that has two stochastic hidden layers. 
We use the widely used binarized MNIST \citep{salakhutdinov2008quantitative} with the standard training/validation/testing partition (50000/10000/10000), making our experimental results directly comparable to previous results in \citep{rws,arm}.

During training, we use the Adam optimizer with learning rate $0.0003$ and minibatch size 50; during testing, we use 1,000 proposal samples to estimate the negative log-likelihood (NLL) for each data point in the test set, as in \citep{vimco,arm}. This setup is also used in section \ref{sec:generative-cate} but with minibatch size 200.
In all experiments in section \ref{sec:generative-ber}, \ref{sec:generative-cate} and \ref{half}, we run different methods on the training set, calculate the validation NLL for every 5 epochs, and report the test NLL when the validation NLL reaches its minimum within a certain number of epochs\footnote{Specifically, 1,000, 500 and 200 epochs are used for experiments in section \ref{sec:generative-ber}, \ref{sec:generative-cate} and \ref{half} respectively.} (i.e. we apply early-stopping).

For training with JSA, theoretically we need to cache a latent $h$-sample for each training data point.
Practically, our JSA method runs in two stages.
In stage I,  we run without caching, i.e. at each iteration, we accept the first proposed sample from $q_\phi(h|x)$ as an initialization and then run MIS with multiple moves.
After stage I, we switch to running JSA in its standard manner.
The idea is that when initially the estimates of $\theta$ and $\phi$ are far from the root of Eq.(\ref{eq:JSA_unsup_obj}), the sample may not be valuable enough to be cached and large randomness could force the estimates moving fast to a neighborhood around the root.
This two-stage scheme yields fast learning while ensuring convergence.
In this experiment, we use the first 600 epochs as stage I and the remaining 400 epochs as stage II.
See \citep{gu2001maximum,tan2017optimally} for similar two-stage SA algorithms and more discussions.

Table \ref{tab:ber} lists the testing NLL results and Figure \ref{fig:ber-NLL} plots training and testing NLL against training epochs and time.
We observe that JSA significantly outperforms other methods 
with faster convergence and lower NLL for both training and test set.
The training time (or complexity) for JSA is comparable to other methods.
Notably, the NLL of JSA drops around epoch 600 when JSA switches from stage I (no caching) and stage II (caching).
This drop indicates the effectiveness of our two-stage scheme and the empirical benefit of our theoretical development (keeping a persistent chain).

\begin{table}[tb]\small
	\caption[Test NLL of different methods with three different network architectures on generative modeling with Bernoulli variables on MNIST, where $*$ denotes the results reported in \citep{arm}, and the others are obtained based on our implementation.
	The mean and standard deviation results are computed over five independent trials with different random seeds.]{Test NLL of different methods with three different network architectures on generative modeling with Bernoulli variables on MNIST, where $*$ denotes the results reported in \citep{arm}, and the others are obtained based on our implementation\footnotemark.
		The mean and standard deviation results are computed over five independent trials with different random seeds.}
	\label{tab:ber}
	\begin{tabular}{l|c|c|c}
		\toprule
		Method&Linear&Nonlinear&Two layers\\
		\midrule
		REINFORCE$^*$&170.1&114.1&159.2\\
		RWS&$108.0\pm0.3$&$99.2\pm0.2$&$96.5\pm0.2$\\
		NVIL$^*$&113.1&102.2&99.8\\
		Concrete$^*$&107.2&99.6&95.6\\
		REBAR$^*$&107.7&100.5&95.5\\
		VIMCO&$107.5\pm0.3$&$100.6\pm0.3$&$95.8\pm0.1$\\
		ARM$^*$&$107.2\pm0.1$&$98.4\pm0.3$&$96.7\pm0.3$\\
		JSA&$105.5\pm0.1$&$98.2\pm0.4$&$95.3\pm0.1$\\
		\bottomrule
	\end{tabular}
\end{table}
\footnotetext{Our Pytorch code and hyperparameters follow \citep{arm}. Some differences are: \citep{arm} uses TensorFlow, runs up to 1,200 epochs, and monitors the validation NLL every one epoch.
	Our run saves time and should be no worse under their conditions.}

\begin{figure*}[tb]
	\subfigure[NLL vs training epochs for ``Linear'']{
		\begin{minipage}{0.48\textwidth}
			\centering
			\includegraphics[width=\textwidth]{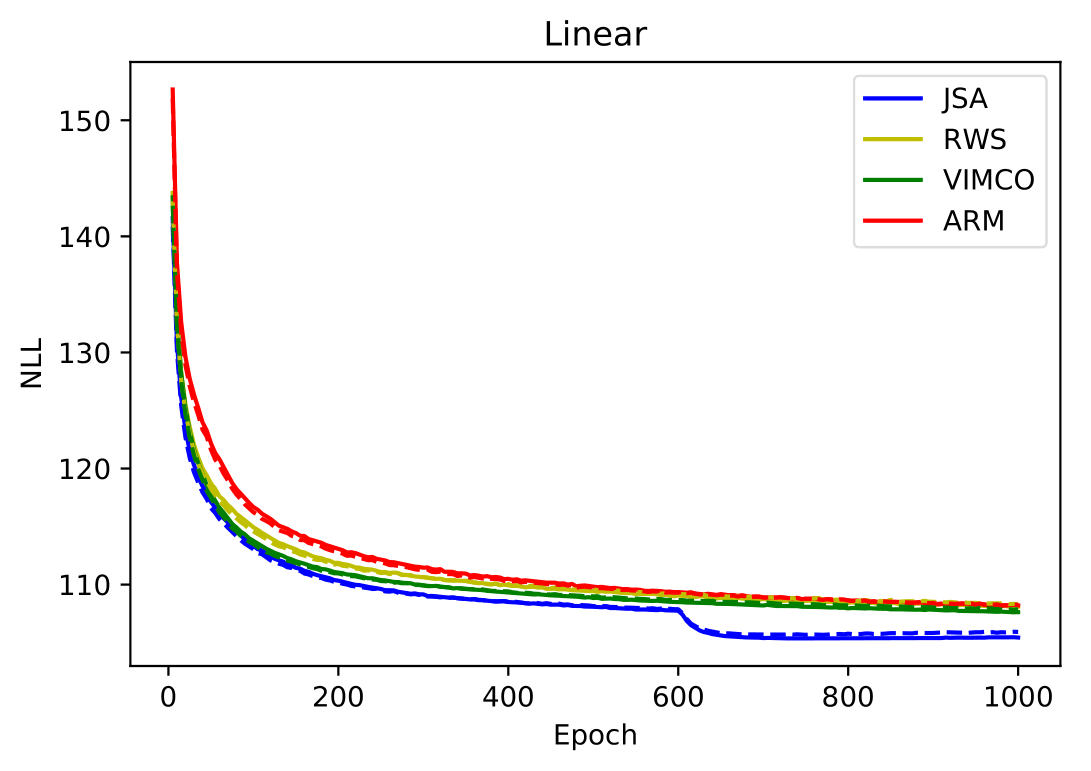} 
	\end{minipage}}
	\subfigure[NLL vs training epochs for ``Nonlinear'']{
		\begin{minipage}{0.48\textwidth}
			\centering
			\includegraphics[width=\textwidth]{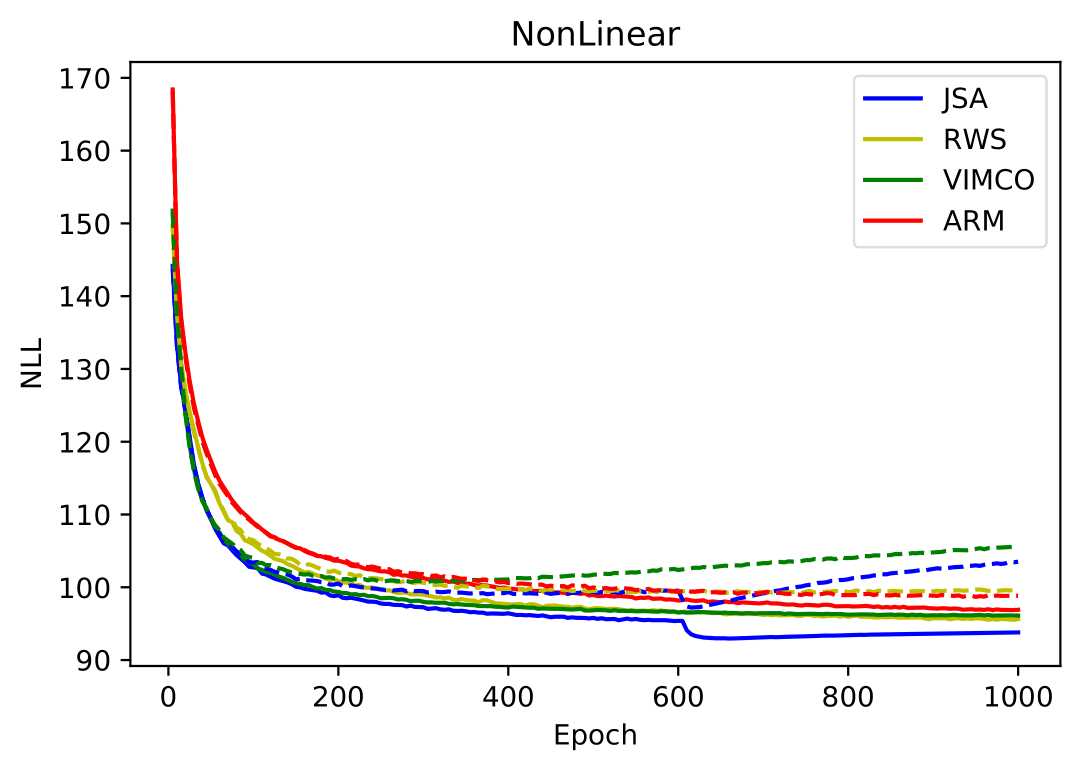} 
	\end{minipage}}
	\subfigure[NLL vs training epochs for ``Two layers'']{
		\begin{minipage}{0.48\textwidth}
			\centering
			\includegraphics[width=\textwidth]{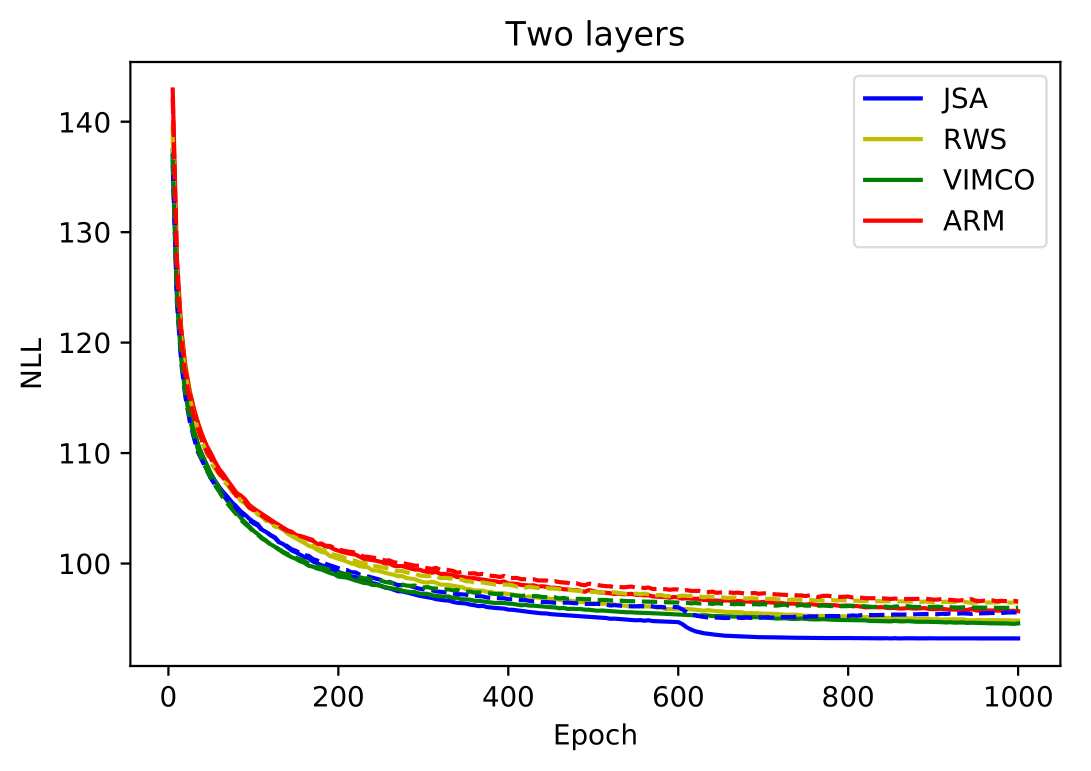} 
	\end{minipage}}
	\subfigure[NLL vs training time for ``Two layers'']{
		\begin{minipage}{0.48\textwidth}
			\centering
			\includegraphics[width=\textwidth]{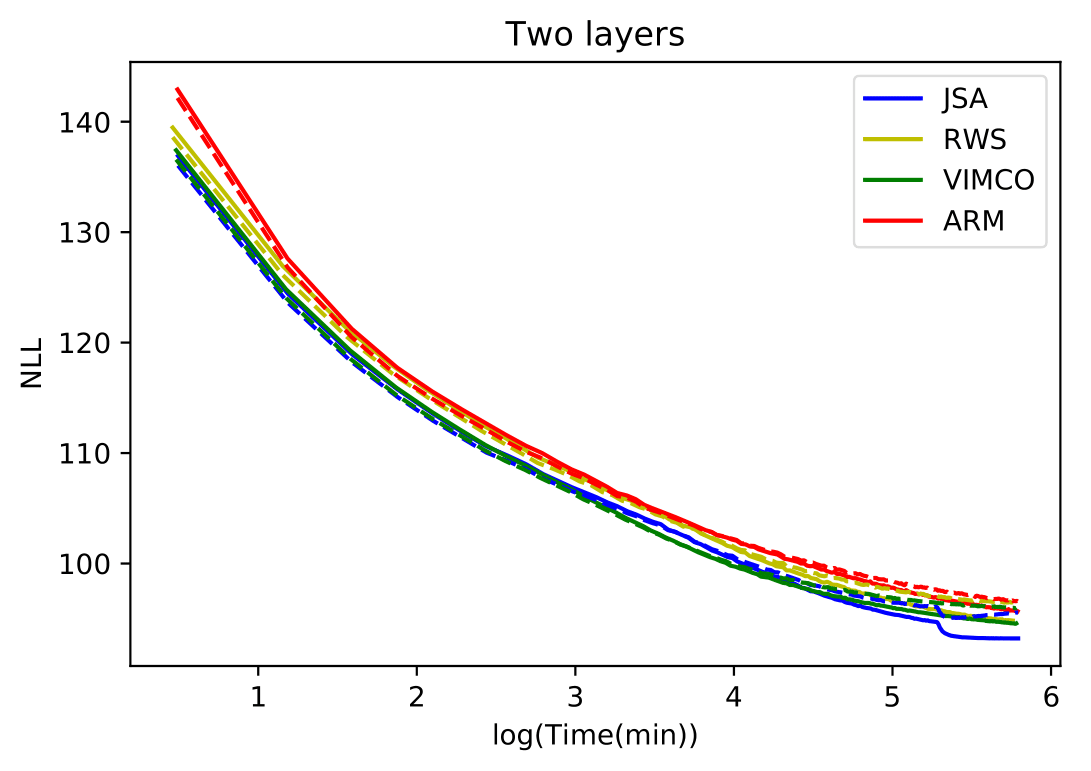} 
	\end{minipage}}
	\caption{Plots of training and testing NLL curves for training models with Bernoulli variables on MNIST.
		(a)(b)(c) are NLL against training epochs for three different network architectures, and (d) is NLL against training wall-clock times for the ``Two layers'' architecture (See Table \ref{fig:ber2} in Appendix for the other two architectures). 
		 The solid and dash lines correspond to the training and testing NLL respectively.
For completeness, we include the testing NLLs against prolonged training epochs in Figure \ref{fig:ber-NLL} and \ref{fig:cate-NLL}, which show overfitting (see Appendix \ref{sec:additional-appendix} for comments).
 }
	\label{fig:ber-NLL}
\end{figure*}

\begin{table*}[tb]
	\caption{Test NLL of different methods on generative modeling of categorical variables on MNIST. The mean and standard deviation results are computed over five independent trials with different random seeds.
	}
\centering
	\label{tab:cate}
	\begin{tabular}{lcccc}
		\toprule
		Method&VIMCO&ST Gumbel-S.&ARSM&JSA\\
		\midrule
		NLL&$79.3\pm0.5$&$82.8\pm0.2$&$78.7\pm0.2$&$75.3\pm0.3$\\
		\bottomrule
	\end{tabular}
\end{table*}

\begin{figure*}[tb]
	\subfigure[NLL vs training epochs]{
		\begin{minipage}{0.48\textwidth}
			\centering
			\includegraphics[width=\textwidth]{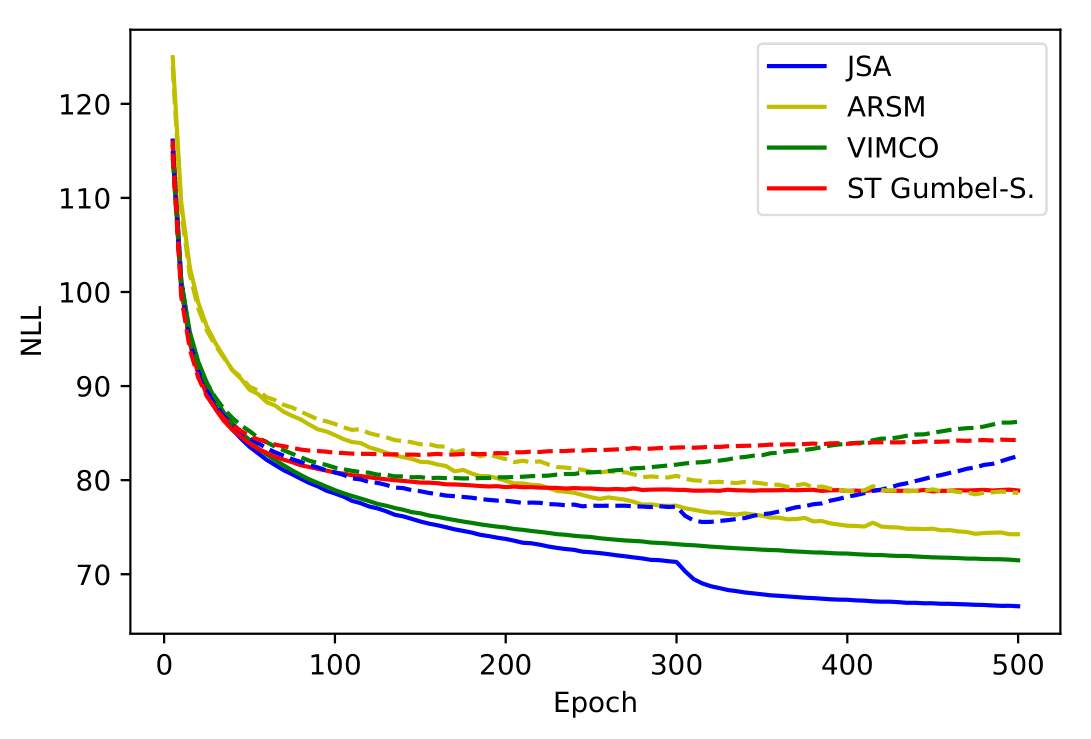} 
	\end{minipage}}
	\subfigure[NLL vs training time]{
		\begin{minipage}{0.48\textwidth}
			\centering
			\includegraphics[width=\textwidth]{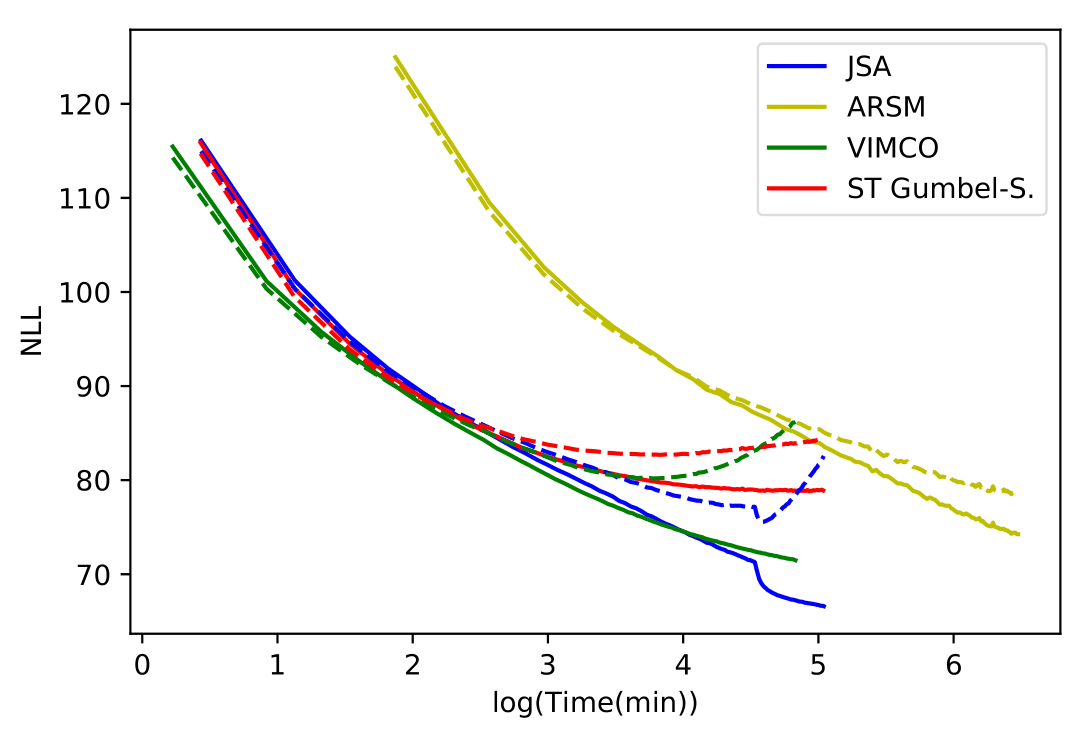} 
	\end{minipage}}
\vspace{-1em}
	\caption{Plots of training and testing NLL curves for training with categorical variables on MNIST, against (a) training epochs, (b) training wall-clock times. The solid and dash lines correspond to the training and testing NLL respectively.}
	\label{fig:cate-NLL}
\end{figure*}

\begin{figure*}[tb]
	\centering
	\subfigure{
		\begin{minipage}{0.48\textwidth}
			\centering
			\includegraphics[width=\textwidth]{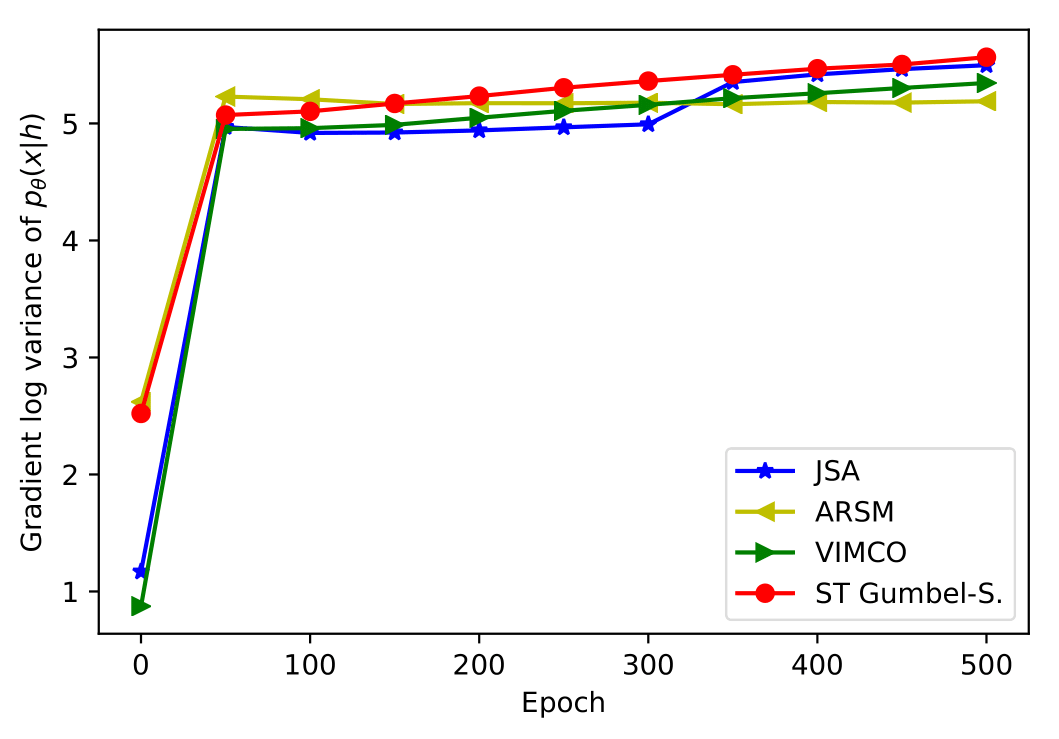} 
	\end{minipage}}
\subfigure{
	\begin{minipage}{0.48\textwidth}
		\centering
		\includegraphics[width=\textwidth]{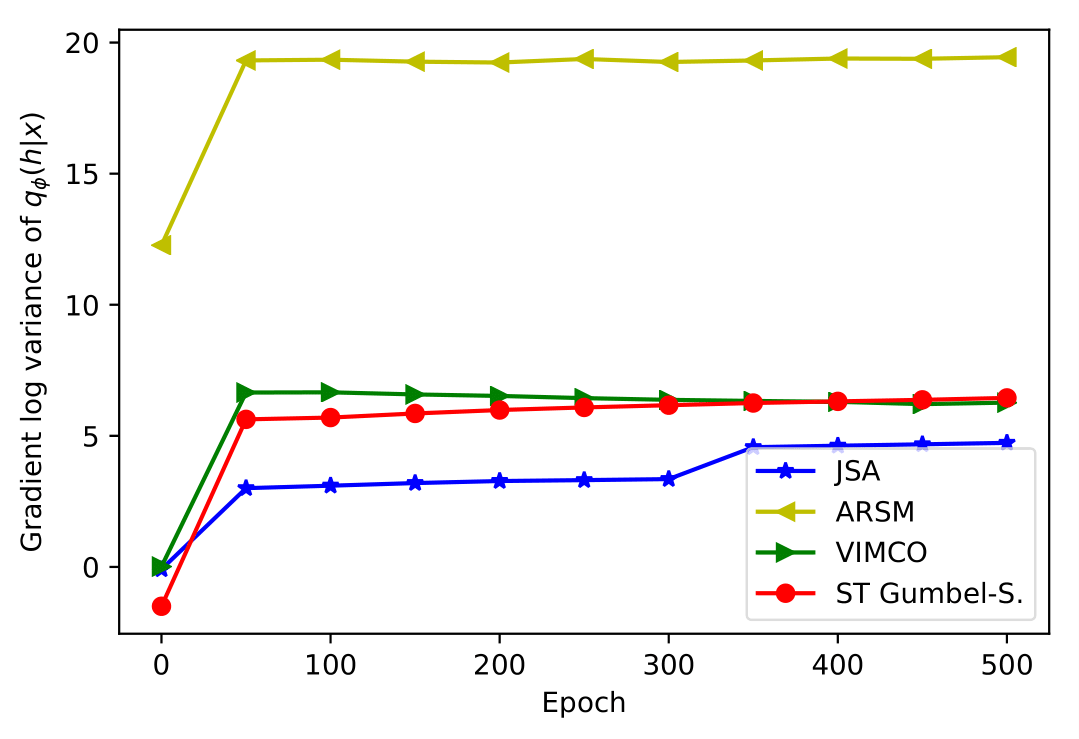} 
\end{minipage}}
\vspace{-0.8em}
	\caption{Log variance of gradient of different methods on generative modeling with categorical variables on MNIST, estimated
		every 50 (out of 500) epochs during training. Each method is computed 1000
		times with different latent samples for the first minibatch (the first 200 records) in the epoch.
		The left and right are variances of gradients w.r.t. $\theta$ and $\phi$ respectively.
		We report (the logarithm of) the sum of the variances per parameter.
	}
	\label{fig:cate-grad}
\end{figure*}

\begin{table}[tb]
	\caption{Test NLL of different methods on structured output prediction with Bernoulli variables on MNIST, where ``n'' denotes \emph{particle-number}. The mean and standard deviation results are computed over five independent trials with different random seeds.
	}
	\label{tab:half}
	\centering
	\begin{tabular}{lccc}
		\toprule
		Method&RWS&VIMCO&JSA\\
		\midrule
		n=5&$46.2\pm0.4$&$46.3\pm0.1$&$45.2\pm0.4$\\
		n=10&$44.7\pm0.1$&$46.2\pm0.2$&$44.2\pm0.1$\\
		n=20&$43.8\pm0.1$&$46.2\pm0.2$&$43.2\pm0.2$\\
		n=40&$43.2\pm0.1$&$46.1\pm0.2$&$42.9\pm0.1$\\
		n=80&$42.8\pm0.0$&$46.0\pm0.1$&$42.6\pm0.1$\\
		\bottomrule
	\end{tabular}
\vspace{-1em}
\end{table}

\subsection{CATEGORICAL LATENT VARIABLES}
\label{sec:generative-cate}
In this experiment, we evaluate various methods for training generative models with categorical latent variables. 
Following the setting of \citep{ARSM}, there are 20 latent variables, each with 10 categories, and the architectures are listed in Table \ref{tab:cate-net} in Appendix.
We compare JSA with other methods that can handle categorical latent variables, including VIMCO \citep{vimco}, straight through Gumbel-Softmax (ST  Gumbel-S.) \citep{gumbel-softmax} and ARSM \citep{ARSM}.
For JSA, ST Gumbel-Softmax and VIMCO, we use \textit{particle-number} $=20$, which yields their theoretical time complexity close to ARSM.
JSA runs with the two-stage scheme, using the first 300 epochs as stage I and the remaining 200 epochs as stage II.

We use the code from \citep{ARSM}, and implement JSA and VIMCO, making the results directly comparable.
We use a binarized MNIST dataset by thresholding each pixel value at 0.5, the same as in \citep{ARSM}. 

Table \ref{tab:cate} lists the test NLL results and Figure \ref{fig:cate-NLL} plots training and testing NLL against training epochs and time.
Similar to results with Bernoulli variables, JSA significantly outperforms other competitive methods.
The training time (or complexity) for JSA is comparable to VIMCO and ST Gumbel-Softmax, and ARSM costs much longer time.
Also, when JSA switches to stage II, a significant decrease of NLL is observed, which clearly shows the benefit of JSA with caching.
Figure \ref{fig:cate-grad} plots the gradient variances of different methods.
The gradient variances w.r.t. $\theta$ for different methods are close and generally smaller than the variances w.r.t. $\phi$. Notably, the gradient variance w.r.t. $\phi$ from JSA is the smallest, which clearly validates the superiority of JSA.

\subsection{STRUCTURED OUTPUT PREDICTION}
\label{half}
Structured prediction is another common benchmark task for  latent variable models. The task is to model a complex observation $x$ given a context $c$, i.e. model the conditional distribution $p(x|c)$. We use a conditional latent variable model $p_\theta(x,h|c)=p_\theta(x|h,c)p_\theta(h|c)$, whose model part is similar to conditional VAE \citep{cvae}.

Specifically, we examine the standard task of modeling the bottom half of a binarized MNIST digit (as $x$) from the top half (as $c$), based on the widely used binarized MNIST \citep{salakhutdinov2008quantitative}.
The latent $h$ consists of 50 Bernoulli variables and the conditional prior $p_\theta(h|c)$ feeds $c$ to two deterministic layers of 200 tanh units to parameterize factorized Bernoulli distributions of $h$.
For $p_\theta(x|h,c)$, we concatenate $h$ with $c$ and feed it to two deterministic layers of 200 tanh units to parameterize factorized Bernoulli outputs.
For $q_\phi(h|x,c)$, we feed the whole MNIST digit to two deterministic layers of 200 tanh units to parameterize factorized Bernoulli distributions of $h$.

RWS, VIMCO, and JSA are conducted with the same models and training setting.
During training, we use Adam optimizer with learning rate $0.0003$ and minibatch size 100.
During testing, we use 1,000 samples from $q_\phi(h|x,c)$ to estimate $-log p_\theta(x|c)$ (NLL) by importance sampling for each data point.
JSA runs with the two-stage scheme, using the first 60 epochs as stage I and afterwards as stage II.

Table \ref{tab:half} lists test NLL results against different number of particles (i.e. the number of Monte Carlo samples used to compute gradients during training).
JSA performs the best consistently under different number of particles.
Both JSA and RWS clearly benefit from using increasing number of particles; but VIMCO not\footnote{Our result of VIMCO here is accordance with the analysis and result of IWAE in \citep{rainforth2018tighter}, which show that using the IW bound with increasing number of particles actually hurt model learning.
	IWAE and VIMCO are two such examples, with continuous and discrete latent variables respectively.}.

\section{CONCLUSION}

We introduce a new class of algorithms for learning latent variable models - JSA.
It directly maximizes the marginal likelihood and simultaneously minimizes the inclusive divergence between the posterior and the inference model. 
JSA couples SA-based model learning and SA-based adaptive MCMC and jointly finds the two sets of unknown parameters for the target latent variable model and the auxiliary inference model.
The inference model serves as an adaptive proposal for constructing the MCMC sampler.
To our knowledge, JSA represents the first method that couples MCMC-SAEM with adaptive MCMC, and encapsulate them through a joint SA procedure.

JSA provides a simple and principled way to handle both discrete and continuous latent variable models.
In this paper, we mainly present experimental results for learning discrete latent variable models with Bernoulli and categorical variables, consisting of stochastic layers or neural network layers.
Our results on several benchmark generative modeling and structured prediction tasks demonstrate that JSA consistently outperforms recent competitive algorithms, with faster convergence, better final likelihoods, and lower variance of gradient estimates.

JSA has wide applicability and is easy to use without any additional parameters, once the target latent variable model and the auxiliary inference model are defined.
The code for reproducing the results in this work is released at \url{https://github.com/thu-spmi/JSA}
\vspace{-1em}
\subsubsection*{Acknowledgements}

Z. Ou is also affiliated with Beijing National Research Center for Information Science and Technology.
This work was supported by NSFC Grant 61976122, China MOE-Mobile Grant MCM20170301. The authors would like to thank Zhiqiang Tan for helpful discussions.	
\vspace{-1em}
\subsubsection*{References}
\bibliographystyle{apalike}
\bibliography{JSA-uai}

\onecolumn

\appendix

\section{On convergence of JSA} 
\label{sec:convergence_jsa}
The convergence of SA has been established under conditions \citep{benveniste2012adaptive, andrieu2005stability, song2014weak}, including a few technical requirements for the mean-field function $f(\lambda)$, the transition kernel $K_{\lambda^{(t-1)}}(z^{(t-1)},\cdot)$ and the learning rates.
We mainly rely on Theorem 5.5 in \citep{andrieu2005stability} to show the convergence of JSA.

For the transition kernel in JSA, it is shown in \citep{Fort2003OnTG} that the random-scan Metropolis-within-Gibbs sampler satisfies the $V$-uniform ergodicity under some mild conditions.
The $V$-uniform ergodicity of the transition kernel is the key property for the transition kernel $K_{\lambda}$ to satisfy the drift condition, which is an important condition used in \citep{andrieu2005stability} to establish the convergence of the SA algorithm.

Specifically, we can apply Theorem 5.5 in Andrieu et al. (2005) to verify the conditions (A1) to (A4) to show JSA convergence. (A1) is the Lyapunov condition on $f(\lambda)$, which typically holds for stochastic optimization problems like in this paper, in which $f(\lambda)$ is a gradient field for some bounded, real-valued and differentiable objective functions. 
(A2) and (A3) hold under the drift condition, which is satisfied by the transition kernel in JSA as outlined above.
(A4) gives conditions on the learning rates, e.g. satisfying that $\sum_{t=0}^\infty \gamma_t = \infty$ and $\sum_{t=0}^\infty \gamma_t^2 < \infty$.

\section{Proof of the Fisher identity} 
\label{sec:Fisher}
Note that $E_{p_\theta(h|x)}\left[ \nabla_\theta logp_\theta(h|x)\right]=0$, so we have $E_{p_\theta(h|x)}\left[ \nabla_\theta logp_\theta(x,h)\right] = E_{p_\theta(h|x)}\left[ \nabla_\theta logp_\theta(x)+logp_\theta(h|x)\right] =\nabla_\theta log p_\theta(x)$.

\section{Additional comments about JSA} \label{sec:additional-appendix}

\textbf{Minibatching in JSA.~}
In our experiments, we run JSA with multiple moves \citep{Wang2017LearningTR} to realize minibatching. Specifically, we draw a minibatch of data points, say, $x_{\kappa_1},\cdots, x_{\kappa_m}$, and for each $x_{\kappa_j}, j=1,\cdots,m$, we generate multiple $h$-samples $h_{\kappa_{j},k}, k=1,\cdots,\textit{particle-number}$. In our pytorch implementation, firstly, we generate proposals for all $h_{\kappa_{j},k}, j=1,\cdots,m, k=1,\cdots,\textit{particle-number}$, according to $q_\phi$ and calculate their importance weights. This can be easily organized in tensor operations. Then we perform accept/reject to obtain the $h$-samples, which takes negligible computation. In this way, JSA can efficiently utilize modern tensor libraries.

\textbf{Comparison with \citep{naesseth2020markovian}.~}
A similar independent work pointed out by one of the reviewers is called Markov score climbing (MSC) \citep{naesseth2020markovian}.
It is interesting that MSC uses the conditional importance sampler (CIS), whereas JSA uses the random-scan sampler. We see two further differences. First, the random-scan sampler in JSA satisfies the $V$-uniform ergodicity, which ensures that the transition kernel satisfies the drift condition for the SA convergence. It is not clear if the CIS sampler satisfies the assumptions listed in \citep{naesseth2020markovian}. Second, the model setups in \citep{naesseth2020markovian} and our work are different for large-scale data. By using the random-scan sampler, JSA can support minibatching in our setup. Minibatching of MSC in the setup of \citep{naesseth2020markovian} leads to systematic errors.

\textbf{Overfitting observed in Figure \ref{fig:ber-NLL} and \ref{fig:cate-NLL}.~}
Similar to previous studies in optimizing discrete latent variable models (e.g. \cite{tucker2017rebar,Kool}), we observe overfitting in Figure \ref{fig:ber-NLL} and \ref{fig:cate-NLL}, when the training process was deliberately prolonged in order to show the effects of different optimization methods for decreasing the training NLLs. 
A common approach (which is also used in our experiments) to address overfitting is monitoring NLL on validation data and apply early-stopping. Thus, for example in Figure \ref{fig:cate-NLL}a, the testing NLL for JSA should be read from the early-stopping point shortly after epoch 300. Different methods early-stop at different epochs, by monitoring the validation NLLs as described in the 3rd paragraph in section \ref{sec:generative-ber}. For example, in training with categorical latent variables (Figure \ref{fig:cate-NLL}), JSA,  ARSM and VIMCO early-stop at epoch 320, 485 and 180 respectively.
These early-stopping results are precisely the testing NLLs reported in Table 1 and 2, from which we can see that JSA significantly outperforms other competitive methods across different latent variable models.

\section{Additional tables and figures} 

\begin{algorithm}[tb]
	\caption{SA with multiple moves}\label{alg:SA-multiple-move}
	\begin{algorithmic}	
		\FOR {$t=1,2,\cdots$}
		\STATE
		\begin{enumerate}
			\item Set $z^{(t,0)}=z^{(t-1,K)}$.
			For $k$ from $1$ to $K$,
			generate $z^{(t,k)} \sim K_{\lambda^{(t-1)}}(z^{(t, k-1)},\cdot)$,
			where $K_{\lambda^{(t-1)}}(z^{(t, k-1)},\cdot)$ is a Markov transition kernel that admits $p_{\lambda^{(t-1)}}(\cdot)$ as the invariant distribution.
			\item Set $\lambda^{(t)} = \lambda^{(t-1)} + \gamma_t \{ \frac{1}{K} \sum_{z\in B^{(t)}} F_{\lambda^{(t-1)}}(z) \}$,  where $B^{(t)} = \{ z^{(t,k)} | k = 1,\cdots,K \}$.
		\end{enumerate}	
		\ENDFOR
	\end{algorithmic}
\end{algorithm}

\begin{table}[h]
	\caption{The three different network architectures in generative modeling with Bernoulli variables on MNIST, which are the same as Table 1 in \citep{arm}. The following symbols ``$\to$'', ``]'', ``)'', and ``$\leadsto$'' represent deterministic linear transform, leaky rectified linear units (LeakyReLU) nonlinear activation, sigmoid nonlinear activation, and random sampling respectively.
	}
	\label{tab:ber-net}
	\begin{tabular}{lccc}
		\toprule
		&Nonlinear&Linear&Linear two layers\\
		\midrule
		$q_\phi(h|x)$&$784\to200]\to200]\to200)\leadsto200$&$784\to200)\leadsto200$&$784\to200)\leadsto200\to200\leadsto200$\\
		$p_\theta(x|h)$&$784\leftsquigarrow(784\leftarrow[200\leftarrow[200\leftarrow200$&$784\leftsquigarrow(784\leftarrow200$&$784\leftsquigarrow(784\leftarrow200\leftsquigarrow(200\leftarrow200$\\
		\bottomrule
	\end{tabular}
\end{table}

\begin{table}[h]
	\caption{Network architectures in generative modeling with categorical variables on MNIST, which are the same as in \citep{ARSM}. The following symbols ``$\to$'', ``]'', ``)'', ``$\leadsto$'' and ``$\hookrightarrow$'' represent deterministic linear transform, LeakyReLU nonlinear activation, sigmoid nonlinear activation, random sampling over Bernoulli distributions and random sampling over categorical distributions respectively.
	For sampling over categorical distributions, there are 20 categorical variables each with 10 categories, thus the 200 output units are divided into 20 groups each with 10 units.
	Then we apply softmax to each group, and obtain the categorical distribution of each variable.
	After obtaining random sample of each variable, we use one-hot encoding and concatenate them to obtain a 200 dimensional array.  
	}
	\label{tab:cate-net}
		\centering
	\begin{tabular}{cc}
	
		\toprule
		$q_\phi(h|x)$&$p_\theta(x|h)$\\
		\midrule
		$784\to512]\to256]\to200\hookrightarrow200$&$784\leftsquigarrow(784\leftarrow[512\leftarrow[256\leftarrow200$\\
		\bottomrule
	\end{tabular}
\end{table}

\begin{figure}[b]
	\centering
	\subfigure[NLL vs training time for ``Linear'']{
		\begin{minipage}{0.48\textwidth}
			\centering
			\includegraphics[width=\textwidth]{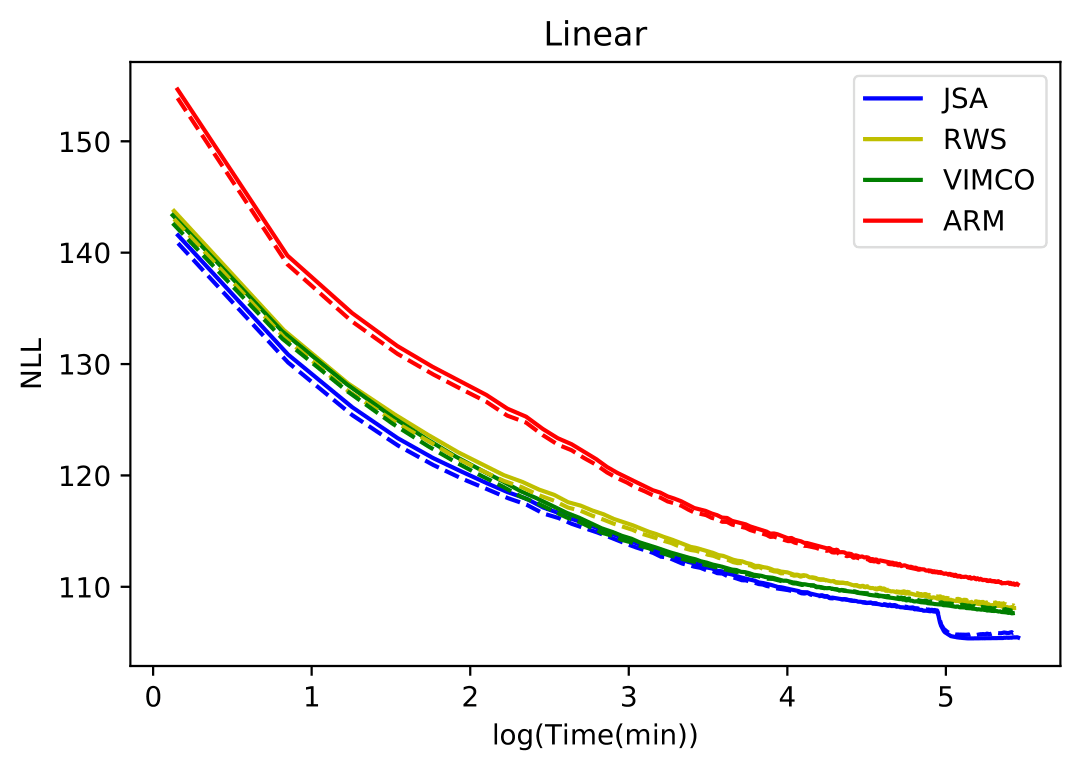} 
	\end{minipage}}
	\subfigure[NLL vs training time for ``Noninear'']{
		\begin{minipage}{0.48\textwidth}
			\centering
			\includegraphics[width=\textwidth]{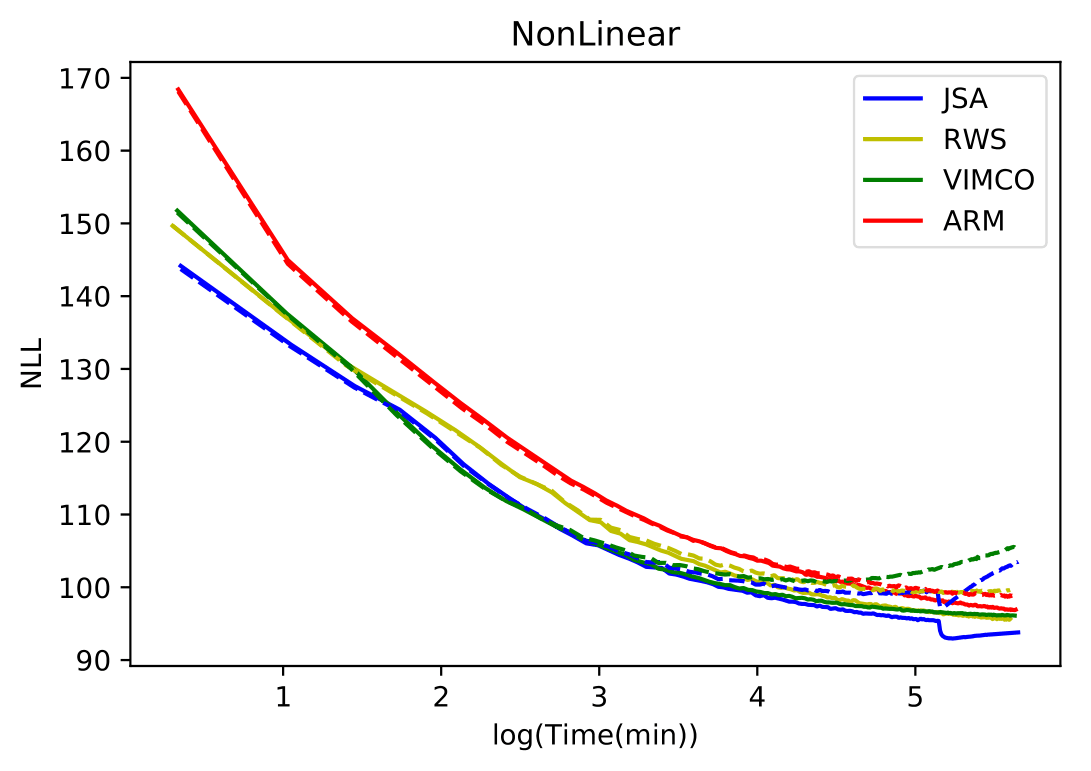} 
	\end{minipage}}
	\caption{Plots of training and testing NLL curves for training models with Bernoulli variables on MNIST, against training wall-clock times. (a)(b) are for ``Linear'' and ``Nonlinear'' architectures respectively. The solid and dash lines correspond to the training and testing NLL respectively.
	}
	\label{fig:ber2}
\end{figure}

\end{document}